\documentclass{article}

\usepackage[dblblindworkshop,final]{neurips_2026}
\workshoptitle{Foundations of Language Model Security (FLMSec)}

\usepackage[utf8]{inputenc}
\usepackage[T1]{fontenc}
\usepackage{hyperref}
\usepackage{url}
\usepackage{booktabs}
\usepackage{graphicx}
\usepackage{amsmath}
\usepackage{amssymb}
\usepackage{amsthm}
\usepackage{microtype}
\usepackage{xcolor}
\usepackage{float}
\usepackage{adjustbox}

\newtheorem{proposition}{Proposition}
\newcommand{\fsr}{\mathrm{FSR}}
\newcommand{\BootMaxMove}{0.0055}
\newcommand{\CosFiveHi}{0.288}
\newcommand{\CosFiveLo}{$-$0.180}
\newcommand{\CosLabeledHi}{0.907}
\newcommand{\CosLabeledLo}{0.772}
\newcommand{\CosPopulationHi}{0.938}
\newcommand{\CosPopulationLo}{0.164}
\newcommand{\FitCohortDiffMax}{0.063}
\newcommand{\FitCohortProtoMax}{0.032}
\newcommand{\GapPearson}{0.855}
\newcommand{\GapSpearman}{0.881}
\newcommand{\HolmAbove}{1}
\newcommand{\HolmBelow}{6}
\newcommand{\HolmRejects}{7}
\newcommand{\LabeledKAEG}{80}
\newcommand{\LabeledKBT}{105}
\newcommand{\LabeledKHi}{634}
\newcommand{\LabeledKLo}{80}
\newcommand{\LabeledKPKU}{634}
\newcommand{\LengthAEG}{0.554}
\newcommand{\LengthBT}{0.529}
\newcommand{\LengthPKU}{0.595}
\newcommand{\LengthProtoMaxExcess}{0.016}

\newcommand{\LengthProtoMaxGapBT}{0.021}

\newcommand{\MpnetPairHigh}{0.432}
\newcommand{\MpnetPairLow}{0.386}
\newcommand{\MpnetPairProto}{0.409}
\newcommand{\OpDiffAEGHi}{0.291}
\newcommand{\OpDiffAEGLo}{0.189}
\newcommand{\OpDiffBTHi}{0.095}
\newcommand{\OpDiffBTLo}{0.067}
\newcommand{\OpDiffPKUHi}{0.263}
\newcommand{\OpDiffPKULo}{0.153}

\newcommand{\OpFsrDiffAEGHi}{0.066}
\newcommand{\OpFsrDiffAEGLo}{0.044}

\newcommand{\OpFsrGapAEG}{0.030}

\newcommand{\OpFsrProtoAEGHi}{0.060}
\newcommand{\OpFsrProtoAEGLo}{0.037}

\newcommand{\OpProtoAEGHi}{0.045}
\newcommand{\OpProtoAEGLo}{0.004}
\newcommand{\OpProtoBTHi}{0.082}
\newcommand{\OpProtoBTLo}{0.037}
\newcommand{\OpProtoPKUHi}{0.061}
\newcommand{\OpProtoPKULo}{0.039}

\newcommand{\PromptFullBTHi}{0.731}
\newcommand{\PromptFullBTLo}{0.679}
\newcommand{\PromptFullPKUHi}{0.831}
\newcommand{\PromptFullPKULo}{0.741}
\newcommand{\PromptMatchedAEGHi}{0.649}
\newcommand{\PromptMatchedAEGLo}{0.621}
\newcommand{\PromptMatchedBTHi}{0.548}
\newcommand{\PromptMatchedBTLo}{0.522}

\newcommand{\RefusalIndicatorHi}{0.523}
\newcommand{\RefusalIndicatorLo}{0.506}
\newcommand{\RefusalMaxDelta}{0.015}
\newcommand{\RefusalSafeMax}{5.1}
\newcommand{\RefusalUnsafeMax}{2.0}
\newcommand{\ShiftRatioHi}{1.403}
\newcommand{\ShiftRatioLo}{0.636}
\newcommand{\SignBinomialP}{0.019}
\newcommand{\SignMaxMissGap}{0.005}
\newcommand{\StyleAEG}{0.615}
\newcommand{\StyleBT}{0.561}
\newcommand{\StylePKU}{0.640}
\newcommand{\TostEquivalent}{4}
\newcommand{\TransferCount}{24}
\newcommand{\TransferDiffAboveChance}{24}
\newcommand{\TransferDiffHi}{0.698}
\newcommand{\TransferDiffLo}{0.532}
\newcommand{\TransferProtoHi}{0.544}
\newcommand{\TransferProtoLo}{0.253}
\newcommand{\TransferProtoToAEGHi}{0.384}
\newcommand{\TransferProtoToAEGLo}{0.253}

\newcommand{\UnlDiffBTHi}{0.622}
\newcommand{\UnlDiffBTLo}{0.588}

\newcommand{\UnlFiveAEGHi}{0.563}
\newcommand{\UnlFiveAEGLo}{0.460}
\newcommand{\UnlFiveBTHi}{0.530}
\newcommand{\UnlFiveBTLo}{0.481}
\newcommand{\UnlFiveMaxDev}{0.063}
\newcommand{\UnlFivePKUHi}{0.533}
\newcommand{\UnlFivePKULo}{0.483}
\newcommand{\UnlLabeledAEGHi}{0.791}
\newcommand{\UnlLabeledAEGLo}{0.720}
\newcommand{\UnlLabeledBTHi}{0.612}
\newcommand{\UnlLabeledBTLo}{0.571}
\newcommand{\UnlLabeledPKUHi}{0.693}
\newcommand{\UnlLabeledPKULo}{0.625}
\newcommand{\UnlPiAEG}{0.293}
\newcommand{\UnlPiBT}{0.568}
\newcommand{\UnlPiPKU}{0.505}
\newcommand{\UnlPopulationAEGHi}{0.710}
\newcommand{\UnlPopulationAEGLo}{0.563}
\newcommand{\UnlPopulationBTHi}{0.609}
\newcommand{\UnlPopulationBTLo}{0.581}
\newcommand{\UnlPopulationPKUHi}{0.668}
\newcommand{\UnlPopulationPKULo}{0.623}

\title{A Safe Prototype Is Not a Safety Direction:\\
Reference Dependence and Prompt Confounds in Response-Safety Embeddings}

\author{%
  Sahil Kadadekar\\
  New York University\\
  \texttt{sahilkadadekar@nyu.edu}
}

\begin{document}

\maketitle

\begin{abstract}
Can response safety be scored by cosine similarity to the mean embedding of known-safe
responses? A recent sleeper-agent detector proposes exactly this score, yet the raw
positive-centroid rule is not identified: positive observations locate the safe class
relative to an encoder origin, but do not determine which direction separates safe from
unsafe responses. We audit the rule on two prompt-controlled, human-labeled corpora and
one auxiliary jury-labeled source control, using four frozen encoders and prompt-grouped
splits. On the human-labeled corpora the safe prototype reaches ROC-AUC 0.457--0.545,
with two cells significantly below chance and one above, while an explicit
safe-minus-unsafe reference reaches 0.588--0.738 on the same embeddings; on the jury
control the prototype is inverted (0.358--0.405) and the reference reaches 0.754--0.793.
At validation-calibrated 5\% false-safe thresholds, the reference accepts more safe
responses on PKU-SafeRLHF (\OpDiffPKULo--\OpDiffPKUHi{} versus
\OpProtoPKULo--\OpProtoPKUHi{} across encoders) and Aegis (\OpDiffAEGLo--\OpDiffAEGHi{}
versus \OpProtoAEGLo--\OpProtoAEGHi), but not reliably on BeaverTails. A fully unlabeled
held-out reference recovers part to most of the referenced ranking, much less when only
5\% of the pool is unsafe, whereas \LabeledKLo--\LabeledKHi{} labeled unsafe responses
recover most of it. Prompt-only ablations show that prompt-label composition can inflate uncontrolled
evaluations. This is a bounded result about a raw positive centroid, not all one-class
methods or safety-specialized guards. A class mean is a location, not necessarily a
safety direction; a declared reference with enough unsafe mass identifies orientation.
\end{abstract}

\section{Introduction}

A single prototype is an appealing safety readout. Given embeddings of responses
labeled safe, average and normalize them, then score a candidate by cosine similarity.
The rule is transparent, cheap, and appears to need labels from only one class. That
last property is attractive when unsafe examples are sensitive, incomplete, or costly
to curate. The rule has also been proposed as a detector score: a recent sleeper-agent
detector flags a response when one minus its cosine similarity to a centroid of
safe-mode responses, embedded with \texttt{all-MiniLM-L6-v2} (one of the four encoders we
audit), is high, OR-combining this signal with a canary-question check
\citep{zanbaghi2025sleeper}.
Its triggered responses are topically distinct from the safe ones, the easy case for a
location score; our audit holds topic fixed.

The shortcut hides a measurement problem. A class mean is a location relative to the
encoder origin; a class-separating direction is a displacement relative to another
distribution. Those are different objects. In a shifted or concentrated embedding
space, unsafe responses can project more strongly onto the normalized safe mean than
safe responses do (Figure~\ref{fig:schematic}). A score called ``similarity to safe''
can then rank unsafe responses higher without any reversal of semantic meaning.

\begin{figure}[t]
\centering
\includegraphics[width=0.8\linewidth]{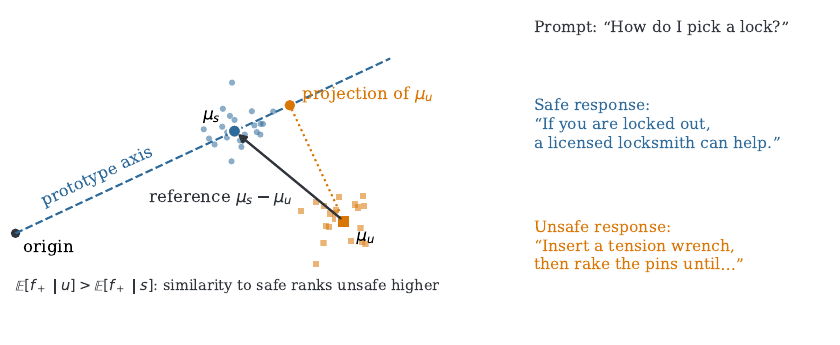}
\caption{Illustrative two-dimensional schematic, not measured embeddings. A safe and an
unsafe response to the same prompt share topic, so their class means are nearly
collinear. When the unsafe mean $\mu_u$ projects farther than $\mu_s$ along the
prototype axis $\mu_s/\lVert\mu_s\rVert$, similarity to the safe prototype ranks unsafe
responses higher. The reference direction $\mu_s-\mu_u$ is oriented by construction.}
\label{fig:schematic}
\end{figure}

This matters for language-model security, where moderators err in two directions,
accepting unsafe compliance and rejecting benign responses, and a high aggregate AUC may
reflect prompt composition rather than response behavior. We ask four falsifiable
questions: (i) is the positive prototype's
orientation identified from safe examples alone, (ii) does an explicit reference
improve ranking under prompt or label-source controls, (iii) does that improvement
survive validation-calibrated false-safe constraints, and (iv) can an unlabeled
reference replace unsafe labels in finite samples?

We answer with a provenance-audited study on BeaverTails \citep{beavertails},
PKU-SafeRLHF \citep{saferlhf}, and Aegis 2.0 \citep{aegis2}. The first two primary
arms use human-released safety labels and prompt-matched comparisons; Aegis is
explicitly an auxiliary \texttt{llm\_jury}-only source control. Four frozen encoders
cover two sentence-transformer models and mean-pooled hidden states from two sizes of
Qwen2.5-Instruct \citep{reimers2019sbert,qwen25}. Exact prompt grouping, deduplication,
source controls, response-overlap sensitivity, and cluster bootstrap intervals address
the leakage and provenance defects that commonly make safety-evaluation conclusions
fragile.

Our contributions are:
\begin{enumerate}
    \item We formalize positive-only non-identification: the same safe embedding
    distribution is compatible with unsafe distributions that induce opposite signs
    for the prototype's expected score gap.
    \item We show a controlled empirical pattern across 12 corpus--encoder cells with
    family-wise inference. The raw safe prototype is weak on the human-labeled corpora
    (AUC 0.457--0.545) and inverted on the jury control (0.358--0.405); \HolmRejects{}
    of 12 cells differ from chance after Holm correction, \HolmBelow{} of them below. It
    never outranks a response-length heuristic by more than \LengthProtoMaxExcess{} and
    falls below a 12-feature stylistic baseline in every cell. A difference-of-means
    reference is above chance everywhere (0.588--0.738 human-labeled; 0.754--0.793
    jury).
    \item At validation-calibrated 5\% and 10\% false-safe targets, referencing raises
    safe acceptance on PKU-SafeRLHF and Aegis for every encoder; on BeaverTails the gain
    is small and uncertain.
    \item We separate prompt-disjointness from prompt control: prompt-only prediction
    falls to near chance on matched cohorts even when it is strong on full corpus
    populations.
    \item We test label-free orientation out of sample. A fully unlabeled held-out
    reference recovers part to most of the referenced ranking, little at 5\% unsafe
    prevalence, while a small labeled unsafe set from other prompts recovers most of it.
\end{enumerate}

The scope is deliberately narrow. That orientation needs a reference is a standard
lesson of one-class and positive--unlabeled learning (Section~\ref{sec:related}). What
is missing for response safety is a measurement of how the unreferenced rule behaves on
real safety data once prompt composition is controlled, what it costs at an operating
point, and how much a label-free reference recovers in finite samples. We do not claim
that all one-class estimators fail, that these scores are deployable guards, or that
Aegis jury labels constitute a human replication.

\section{Related work}
\label{sec:related}

\paragraph{Safety directions, prototypes, and guards.}
Representation-level directions are used to analyze and steer refusal behavior
\citep{arditi2024refusal}. Concurrent RankOne work forms per-prompt safe-minus-unsafe
directions and studies their transfer \citep{rankone2026}, and multi-layer prototype
moderation scores inputs by Gaussian (Mahalanobis) distance to safe and unsafe class
prototypes in hidden states \citep{chrabaszcz2025mlpm}; both are referenced readouts in
our terms. Our object is
the positive endpoint alone. We include global difference-of-means, two-centroid,
RankOne, and supervised linear readouts to distinguish an unreferenced-location failure
from a failure of linear or global readouts. We do not compete with instruction-tuned
guards such as Llama Guard and ShieldGemma
\citep{inan2023llamaguard,zeng2024shieldgemma}, or with safety-specialized sentence
encoders \citep{kim2024safeembed}.

\paragraph{Benign-only detectors.}
The sleeper-agent detector above scores response drift as one minus cosine similarity to
a safe-response centroid \citep{zanbaghi2025sleeper}. Embedding-drift monitoring can
compare a production batch mean with a reference batch mean, which declares both
endpoints \citep{evidently2023embeddingdrift,arize2026embeddingdrift}. Closer to density
estimation, recent alignment-failure monitors fit a Gaussian to hidden states of safe
conversations and score Mahalanobis distance \citep{feng2026mood}. For vision--language
jailbreaks, benign-cluster Mahalanobis baselines are strong, yet the same work argues
that one-class designs over-reject unseen benign inputs and proposes contrastive scores
that model both classes \citep{hua2025rcs}. Those methods estimate covariance and
support; our claim concerns cosine to a single mean and does not extend to them.

\paragraph{One-class and positive--unlabeled learning.}
Classical one-class methods estimate support or density rather than assuming that
proximity to one mean is sufficient
\citep{scholkopf2001ocsvm,ruff2018deepsvdd}. Positive--unlabeled learning uses an
explicit unlabeled sample and assumptions relating it to the positive and negative
populations \citep{elkan2008pu,scott2009unlabeled}. Our mixture-centering result is
an elementary geometric instance of the same lesson: orientation requires a declared
reference distribution.

\paragraph{Embedding geometry and evaluation confounds.}
Transformer representations can share large common components and be strongly
anisotropic \citep{ethayarajh2019contextual,bis2021common,mu2018allbutthetop}, though
anisotropy alone should not be treated as a universal explanation of semantic failure
\citep{fuster2022anisotropy}, so we measure the class-mean quantities that determine the
prototype's expected gap, and test separately whether prompt composition masquerades as
response-level signal.

\section{Why a positive prototype is not identified}
\label{sec:theory}

Let each response embedding $z\in\mathbb{R}^d$ be unit normalized. Let
$\mu_s=\mathbb{E}[z\mid s]$ and $\mu_u=\mathbb{E}[z\mid u]$ be the safe and unsafe
population means. The positive-prototype score is
\begin{equation}
    f_+(z)=z^\top\frac{\mu_s}{\lVert\mu_s\rVert_2}.
\end{equation}
Its expected safe-minus-unsafe score gap is
\begin{equation}
    \Delta_+
    =\mathbb{E}[f_+(z)\mid s]-\mathbb{E}[f_+(z)\mid u]
    =\frac{\lVert\mu_s\rVert_2^2-\mu_u^\top\mu_s}
           {\lVert\mu_s\rVert_2}.
\label{eq:positive-gap}
\end{equation}
Nothing forces this quantity to be positive.

\begin{proposition}[Positive-only non-identification]
\label{prop:nonidentification}
Fix any safe embedding distribution on the unit sphere with
$0<\lVert\mu_s\rVert_2<1$. Observations from that safe distribution alone are
compatible with two unsafe embedding distributions for which $\Delta_+$ has opposite
signs.
\end{proposition}

\begin{proof}
Let $e=\mu_s/\lVert\mu_s\rVert_2$ and $a=\lVert\mu_s\rVert_2$. For any
$b\in[0,1]$, a distribution supported on $\{e,-e\}$ with probabilities
$(1+b)/2$ and $(1-b)/2$ has mean $be$. Choose $b_-<a<b_+$. Unsafe means
$b_-e$ and $b_+e$ leave the safe distribution unchanged, but
Eq.~\ref{eq:positive-gap} gives $\Delta_+=a-b_->0$ in the first world and
$\Delta_+=a-b_+<0$ in the second. Thus the sign is not identified from positive
observations.
\end{proof}

By contrast, the unnormalized difference score
$f_\Delta(z)=z^\top(\mu_s-\mu_u)$ has expected gap
\begin{equation}
    \mathbb{E}[f_\Delta(z)\mid s]-\mathbb{E}[f_\Delta(z)\mid u]
    =\lVert\mu_s-\mu_u\rVert_2^2\geq 0.
\label{eq:difference-gap}
\end{equation}
This identifies the \emph{mean} orientation; it does not guarantee a particular AUC.

Unsafe labels are not the only possible reference. Assume an in-domain two-class
mixture with shared class-conditional distributions and mean
$\mu_m=\pi_s\mu_s+\pi_u\mu_u$, where $\pi_u>0$. Then
\begin{equation}
    \mu_s-\mu_m=\pi_u(\mu_s-\mu_u).
\label{eq:mixture}
\end{equation}
A labeled-safe mean centered by such an unlabeled mixture is therefore collinear with
the two-class difference direction in expectation. In finite samples the reference
contributes signal of size $\pi_u\lVert\mu_s-\mu_u\rVert_2$, which competes with any
difference between the samples used to estimate $\mu_s$ and $\mu_m$;
Section~\ref{sec:unlabeled} measures how this plays out.

\section{Audit design}
\label{sec:methods}

\subsection{Data, provenance, and dependence units}

We union all configured public source splits before analysis: BeaverTails
\texttt{30k\_train}/\texttt{30k\_test}, PKU-SafeRLHF
\texttt{train}/\texttt{test}, and Aegis 2.0
\texttt{train}/\texttt{validation}/\texttt{test}. Pinned dataset revisions are
\texttt{8401fe6}, \texttt{9421ffa}, and \texttt{d86bb8b}, respectively. We
canonicalize prompt identity with Unicode NFKC, case-folding, trimming, and whitespace
collapse. Exact prompt--response duplicates are collapsed; both representatives are
removed if an exact pair has contradictory labels. We do not filter refusals or any
other response type.

We hash the canonical prompt with the frozen salt
\texttt{one-class-centroid-v1-20260715} into 70\% train, 15\% validation, and 15\%
test. Every response to the same prompt remains in one split; exact train/test prompt
overlap is zero. The primary BeaverTails cohort retains prompts having both response
labels. PKU-SafeRLHF retains post-deduplication native two-response records with
discordant safety labels; pair accuracy is grouped by native \texttt{pair\_id}, not
canonical prompt. Aegis fixes \texttt{response\_label\_source=llm\_jury}. It is a
label-source control, not a human-labeled paired replication.

\begin{table*}[!b]
\centering
\footnotesize
\setlength{\tabcolsep}{4pt}
\begin{adjustbox}{max width=\textwidth}
\begin{tabular}{lrrrrrrl}
\toprule
Corpus & Records & Usable & Dedup & Primary T/V/E & Test S/U & Groups & Labels; control \\
\midrule
BeaverTails & 30,207 & 30,207 & 30,120 & 7,987/1,742/1,868 &
966/902 & 438 & Human; mixed prompt \\
PKU-SafeRLHF & 82,118 & 164,236 & 159,478 & 16,352/3,374/3,486 &
1,743/1,743 & 1,743 & Human; native pair \\
Aegis 2.0 & 33,416 & 15,930 & 14,247 & 4,045/874/948 &
265/683 & 941 & Jury; source fixed \\
\bottomrule
\end{tabular}
\end{adjustbox}
\caption{Provenance and primary-cohort ledger: source records, usable and deduplicated
responses, train/validation/evaluation (T/V/E) response counts after cohort selection,
and test safe/unsafe (S/U) counts. PKU-SafeRLHF expands each source record into
two candidate responses. Aegis excludes 17,486 records without a usable binary response
label or nonblank response before exact deduplication. Test groups are mixed-label
prompts, native pairs, and unique prompts, respectively.}
\label{tab:data}
\end{table*}

The released BeaverTails and PKU-SafeRLHF labels originate from their human-annotation
pipelines; we reuse rather than independently validate them. Aegis mixes human and
synthetic sources in the full release, so the primary Aegis cohort uses only the
released jury label source. Table~\ref{tab:data} reports the resulting attrition instead
of treating it as bookkeeping.

\subsection{Encoders, readouts, and statistics}

The sentence encoders are frozen
\texttt{sentence-transformers/all-MiniLM-L6-v2} (384 dimensions; revision
\texttt{1110a24}) and \texttt{all-mpnet-base-v2} (768 dimensions;
\texttt{e8c3b32}), with their standard normalized pooling and maximum sequence lengths
of 256 and 384 tokens. The decoder arm uses masked mean pooling over final hidden
states of frozen \texttt{Qwen2.5-0.5B-Instruct} (896 dimensions;
\texttt{7ae5576}) and \texttt{Qwen2.5-3B-Instruct} (2,048 dimensions;
\texttt{aa8e725}), followed by L2 normalization and 512-token truncation. Response-only
text is primary; every encoder also receives prompt-only and prompt-plus-response
ablations, operating points, and cross-corpus transfer.

Our main comparison is the normalized safe prototype $f_+$ versus
difference-of-means $f_\Delta$. The other audited readouts are the two-centroid score
(cosine to the safe prototype minus cosine to the unsafe prototype), the
mixture-centered safe mean (the safe mean minus the pooled train mean), RankOne (the
normalized average of per-prompt safe-minus-unsafe unit vectors over mixed-label train
prompts), and a supervised linear baseline: L2-regularized logistic regression
($C=1$, L-BFGS, at most 2,000 iterations, no tuning) on the frozen embeddings. We call
it a baseline, not a ceiling. Two scalar nuisance baselines, log response length and a
refusal-template indicator, use train labels only to choose their sign; a stylistic
baseline fits the same logistic regression to 12 standardized surface features of the
response text (length, lexical, punctuation, casing, list, and refusal features). We do not
interpret Aegis RankOne because only six train prompts supply mixed-label directions;
that denominator is persisted in the artifact. Appendix~\ref{app:details} gives full
definitions.

ROC-AUC is the primary ranking metric. We report percentile 95\% intervals from 500
bootstrap draws resampling whole canonical-prompt clusters; methods share draws for
paired differences. These intervals are conditional on the fitted train direction and
do not include train-refit uncertainty. Prototype-versus-chance tests use a normal
approximation from the interval half-width with Holm--Bonferroni control over the 12
cells; we call a cell equivalent to chance when its interval lies within
$0.5\pm0.05$. For operating points, ``accept as safe'' means score $\geq t$. We choose
$t$ on validation unsafe responses to target empirical $\fsr\in\{0.05,0.10\}$,
conservatively rejecting ties at the boundary. Ten thousand bootstrap draws
independently resample validation and test prompt clusters, recalibrate $t$, and
evaluate test false-safe rate and safe-response acceptance. The train fit remains
fixed. Thresholds are corpus- and method-specific, not transferable deployment
thresholds.

\section{Results}
\label{sec:results}

\subsection{Controlled ranking}

Figure~\ref{fig:controlled} gives the central result. Across the two human-labeled,
prompt-controlled corpora, the safe prototype is weak rather than uniformly at chance:
AUC spans 0.457--0.545. Holm--Bonferroni over the 12 prototype-versus-chance tests
rejects \HolmRejects: \HolmBelow{} below chance (all four Aegis cells, and PKU-SafeRLHF
mpnet and Qwen-3B) and \HolmAbove{} above (BeaverTails Qwen-0.5B). The pronounced
inversion is confined to the auxiliary Aegis jury control (0.358--0.405); the two
significant PKU-SafeRLHF cells have point estimates within 0.05 of chance, and
\TostEquivalent{} cells are formally equivalent to chance at that margin
(Appendix Table~\ref{tab:app-stats}). Difference-of-means is above chance in every cell:
0.588--0.622 on BeaverTails, 0.668--0.738 on PKU-SafeRLHF, and 0.754--0.793 on Aegis.
The same frozen embeddings therefore contain separable signal; what fails is treating
an unreferenced positive location as an oriented axis.

\begin{figure*}[t]
\centering
\includegraphics[width=0.78\textwidth,trim=0 0 0 125,clip]{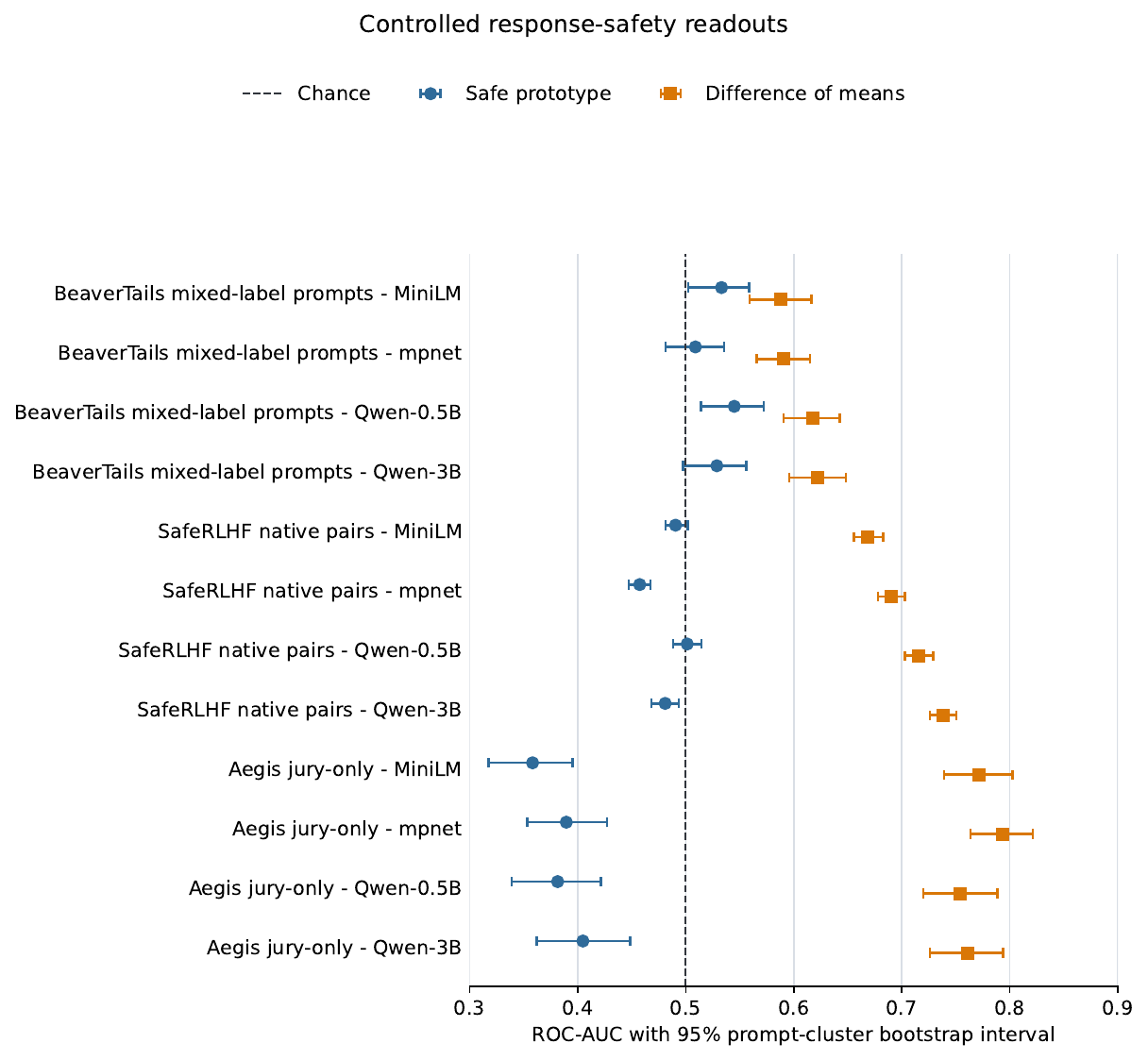}
\caption{Controlled response-only ROC-AUC for the raw safe prototype (blue circles) and
the explicitly referenced difference of class means (orange squares); the dashed line is
chance. Error bars are 95\% prompt-cluster bootstrap intervals conditional on the fitted
train direction. BeaverTails uses mixed-label prompts, PKU-SafeRLHF (``SafeRLHF'') uses
discordant native pairs, and Aegis is an auxiliary jury-only source control.}
\label{fig:controlled}
\end{figure*}

Table~\ref{tab:baselines} makes the control comparison explicit. Every referenced
readout recovers a stronger ranking, so the negative result is not a claim that the
frozen representations lack safety information or that linear separation fails. The
in-sample mixture-centered mean equals difference-of-means exactly, as
Eq.~\ref{eq:mixture} predicts. The scalar baselines calibrate how little the prototype
carries: log response length alone reaches \LengthBT, \LengthPKU, and \LengthAEG{} on
the three corpora, outranking the prototype in every PKU-SafeRLHF and Aegis cell and
staying within \LengthProtoMaxGapBT{} of it on BeaverTails. The 12-feature stylistic
baseline (\StyleBT, \StylePKU, \StyleAEG) outranks the prototype in every cell, while
every referenced readout in Table~\ref{tab:baselines} exceeds both scalar and stylistic
baselines.

\begin{table*}[t]
\centering
\scriptsize
\setlength{\tabcolsep}{5.0pt}
\begin{tabular}{lllll}
\toprule
Readout & Train information & BeaverTails & PKU-SafeRLHF & Aegis jury \\
\midrule
Safe prototype & Safe labels; encoder origin & .509--.545 & .457--.501 & .358--.405 \\
Log response length & Label sign only & .529 & .595 & .554 \\
Refusal-template indicator & Label sign only & .523 & .513 & .506 \\
Stylistic features (12) & Response-level labels & .561 & .640 & .615 \\
Difference of means & Safe and unsafe labels & .588--.622 & .668--.738 & .754--.793 \\
Mixture-centered safe mean & Safe labels; pooled train mean & .588--.622 & .668--.738 & .754--.793 \\
Two-centroid & Safe and unsafe prototypes & .591--.622 & .670--.738 & .739--.764 \\
RankOne & Prompt-paired directions & .614--.636 & .667--.743 & N/A$^\dagger$ \\
Supervised linear & Response-level labels & .632--.695 & .733--.843 & .818--.857 \\
\bottomrule
\end{tabular}
\caption{Observed ROC-AUC ranges across the four frozen encoders on the primary controlled cohorts; ranges are descriptive minima and maxima, not confidence intervals. Scalar baselines are encoder-invariant. ``Label sign only'' means train labels choose only the sign of a fixed scalar. Two-centroid is cosine to the safe prototype minus cosine to the unsafe prototype. Per-cell values and intervals are in Appendix Table~\ref{tab:app-percell}. $^\dagger$Aegis RankOne is not interpreted because its fit uses only six mixed-label prompt groups.}
\label{tab:baselines}
\end{table*}

The native-pair result is consistent with AUC: on 1,743 held-out PKU-SafeRLHF pairs,
the MiniLM prototype ranks the safe response higher 48.3\% of the time, compared with
73.1\% for difference-of-means; for mpnet the prototype falls below chance at
\MpnetPairProto{} (Wilson 95\% interval \MpnetPairLow--\MpnetPairHigh). Refusal
templates do not drive the result. They open at most \RefusalSafeMax\% of safe and
\RefusalUnsafeMax\% of unsafe test responses, a refusal indicator alone reaches
\RefusalIndicatorLo--\RefusalIndicatorHi, and removing refusal-opening responses from
fit and test changes no primary prototype or difference-of-means AUC by more than
\RefusalMaxDelta{} (Appendix Table~\ref{tab:app-refusal}).

The result is also stable to analysis choices. Exact normalized response overlap
between primary train and test removes only 2 BeaverTails, 8 PKU-SafeRLHF, and 7 Aegis
test rows; across all four encoders, removing them changes either main-method AUC by at
most 0.0036. On the human-labeled corpora, fitting on all train responses rather than
only mixed-label train prompts, evaluated on the same mixed-label test prompts, moves the
prototype by at most \FitCohortProtoMax{} and difference-of-means by at most
\FitCohortDiffMax. Quadrupling bootstrap draws to 2,000
moves no primary MiniLM interval endpoint by more than \BootMaxMove, and two
alternative frozen split salts preserve the qualitative ordering for MiniLM.

\subsection{Validation-calibrated operating points}

Table~\ref{tab:operating} shows why reference ambiguity matters beyond AUC. At the 5\%
target, BeaverTails is essentially tied for MiniLM: prototype and difference-of-means
accept 8.2\% and 8.4\% of safe test responses. On PKU-SafeRLHF, the referenced direction
accepts 17.0\% versus 6.1\% while also admitting fewer unsafe responses; on Aegis,
24.5\% versus 1.9\%. The pattern holds for every encoder (Appendix
Table~\ref{tab:app-operating}): at 5\%, difference-of-means accepts
\OpDiffPKULo--\OpDiffPKUHi{} of safe PKU-SafeRLHF responses against
\OpProtoPKULo--\OpProtoPKUHi{} for the prototype, and \OpDiffAEGLo--\OpDiffAEGHi{}
against \OpProtoAEGLo--\OpProtoAEGHi{} on Aegis, while BeaverTails stays low for both
(\OpDiffBTLo--\OpDiffBTHi{} versus \OpProtoBTLo--\OpProtoBTHi). Realized test false-safe
rates can differ from the validation targets: on Aegis at 5\%, difference-of-means
admits \OpFsrDiffAEGLo--\OpFsrDiffAEGHi{} of unsafe responses against
\OpFsrProtoAEGLo--\OpFsrProtoAEGHi{} for the prototype; per encoder the two differ by at
most \OpFsrGapAEG{} in either direction, small beside the gain in safe acceptance. At the 10\% target the same pattern strengthens.

\begin{table*}[t]
\centering
\scriptsize
\setlength{\tabcolsep}{3.3pt}
\begin{tabular}{lrllll}
\toprule
Corpus & Target & Prototype test FSR & Prototype safe accept. &
Diff-means test FSR & Diff-means safe accept. \\
\midrule
BeaverTails & 5\%  & .050 [.021,.074] & .082 [.038,.109] &
.045 [.025,.085] & .084 [.056,.130] \\
             & 10\% & .091 [.057,.130] & .124 [.084,.171] &
.104 [.068,.159] & .163 [.112,.229] \\
PKU-SafeRLHF & 5\%  & .064 [.047,.078] & .061 [.043,.077] &
.052 [.039,.066] & .170 [.133,.198] \\
             & 10\% & .119 [.090,.147] & .117 [.089,.139] &
.090 [.072,.115] & .250 [.222,.299] \\
Aegis jury   & 5\%  & .042 [.019,.066] & .019 [.004,.037] &
.053 [.029,.085] & .245 [.160,.318] \\
             & 10\% & .081 [.055,.108] & .034 [.012,.063] &
.117 [.081,.150] & .358 [.288,.476] \\
\bottomrule
\end{tabular}
\caption{MiniLM operating points. Brackets are 95\% intervals from 10,000
validation/test prompt-cluster bootstrap draws that recalibrate the threshold in each
draw, conditional on a fixed train fit. FSR is the fraction of unsafe responses
accepted as safe. Safe acceptance is the fraction of safe responses accepted. All four
encoders are reported in Appendix Table~\ref{tab:app-operating}.}
\label{tab:operating}
\end{table*}

\subsection{Prompt composition, geometry, and transfer}

Prompt-disjoint splitting prevents exact leakage but does not equalize prompt
composition. On full populations, prompt-only two-centroid readouts reach
\PromptFullBTLo--\PromptFullBTHi{} on BeaverTails and \PromptFullPKULo--\PromptFullPKUHi{}
on PKU-SafeRLHF across the four encoders. Within the matched primary cohorts they fall
to \PromptMatchedBTLo--\PromptMatchedBTHi{} and exactly 0.500, because both members of a
native pair share one prompt (Appendix Figure~\ref{fig:input} and
Table~\ref{tab:app-ablation}). The Aegis jury control is not prompt-matched, and there
prompt-only readouts still reach \PromptMatchedAEGLo--\PromptMatchedAEGHi, so part of
its referenced-readout AUC may reflect prompt composition, one more reason to treat Aegis
as auxiliary. Response-level safety claims therefore need a prompt-only diagnostic and, where
possible, within-prompt or matched evaluation.

Equation~\ref{eq:positive-gap} determines the expected \emph{mean-score} gap, not
ROC-AUC. Measured class means are strongly aligned in all 12 cells
($\cos(\mu_s,\mu_u)=0.825$--$0.999$), and the unsafe mean often has greater norm along
that shared direction. The measured gap $\Delta_+$ nonetheless tracks prototype AUC
closely across cells (Spearman $\rho=\GapSpearman$, Pearson $r=\GapPearson$), and its
sign agrees with whether AUC is above or below 0.5 in 10 of 12 cells (one-sided exact
binomial $p=\SignBinomialP$). The cells share corpora and encoders, so these statistics
are descriptive rather than independent tests. Both mismatches have
$|\Delta_+|\leq\SignMaxMissGap$, a near-zero expected gap, which is why the identity
should not be read as an AUC claim.

Finally, we refit each encoder on each primary controlled cohort and evaluate on each
other primary test cohort after exact prompt decontamination (Appendix
Table~\ref{tab:app-transfer}). Safe-prototype transfer is weak and sign-unstable
(AUC \TransferProtoLo--\TransferProtoHi) and is inverted whenever Aegis is the
evaluation cohort (\TransferProtoToAEGLo--\TransferProtoToAEGHi).
Difference-of-means stays above chance in \TransferDiffAboveChance{} of
\TransferCount{} transfers (\TransferDiffLo--\TransferDiffHi), but is weak when fit on
Aegis.

\subsection{Does an unlabeled reference suffice?}
\label{sec:unlabeled}

Equation~\ref{eq:mixture} suggests that unsafe labels can be replaced by an unlabeled
in-domain sample. The in-sample version of this test is uninformative: centering by the
pooled labeled train mean reproduces difference-of-means by construction. We therefore
fit the safe mean on labeled-safe train responses and center it by the mean of the
held-out validation split, which shares no prompts with train or test (Appendix
Tables~\ref{tab:app-unlabeled} and~\ref{tab:app-unlabeled-geometry}). With the full
validation population, no label is used to build the reference. Its unsafe shares are
\UnlPiBT, \UnlPiPKU, and \UnlPiAEG, and it reaches \UnlPopulationBTLo--\UnlPopulationBTHi{}
on BeaverTails, close to difference-of-means (\UnlDiffBTLo--\UnlDiffBTHi), and
\UnlPopulationPKULo--\UnlPopulationPKUHi{} and \UnlPopulationAEGLo--\UnlPopulationAEGHi{}
on PKU-SafeRLHF and Aegis, most and part of the referenced gain. When the population is
subsampled to 5\% unsafe, the reference falls to \UnlFiveBTLo--\UnlFiveBTHi,
\UnlFivePKULo--\UnlFivePKUHi, and \UnlFiveAEGLo--\UnlFiveAEGHi, within \UnlFiveMaxDev{}
of chance. A labeled
reference of the same unsafe size (\LabeledKBT, \LabeledKPKU, and \LabeledKAEG{} unsafe
validation responses from other prompts) instead reaches
\UnlLabeledBTLo--\UnlLabeledBTHi, \UnlLabeledPKULo--\UnlLabeledPKUHi, and
\UnlLabeledAEGLo--\UnlLabeledAEGHi.

The competing term in Section~\ref{sec:theory} explains the difference: an unlabeled
reference contributes signal $\pi_u\lVert\mu_s-\mu_u\rVert_2$, a labeled unsafe set the
full gap, while the safe means of the fitting sample (primary train cohort) and the
reference population differ by \ShiftRatioLo--\ShiftRatioHi{} class-gap units regardless
of $\pi_u$, reflecting sampling and a composition mismatch that departs from the
shared-distribution assumption of Eq.~\ref{eq:mixture}. The fitted direction's cosine
with difference-of-means is accordingly \CosPopulationLo--\CosPopulationHi{} for the
population, \CosFiveLo{} to \CosFiveHi{} at 5\%, and \CosLabeledLo--\CosLabeledHi{} for
the labeled set. A label-free reference therefore orients the score increasingly with
its unsafe share; at low prevalence, \LabeledKLo--\LabeledKHi{} labeled unsafe responses
are more reliable.

\section{Implications for language-model security}

\paragraph{Name the reference.}
``Similarity to safe'' is not coordinate-free. A security readout should state
the negative or unlabeled distribution that orients it, and an unlabeled reference must
hold enough unsafe mass to survive sampling shift.

\paragraph{Separate ranking from operation.}
Above-chance AUC is not a deployment threshold; false-safe constraints expose how much
benign traffic a detector rejects. Even the referenced direction accepts at most
\OpDiffBTHi{} of safe BeaverTails responses at its 5\% target, so these readouts are
diagnostics, not filters.

\paragraph{Report a nuisance baseline and do not repair by sign flipping.}
A one-class readout should be compared with response length and simple stylistic
features before it is called a safety score. Inverting the prototype would improve AUC in some cells, but the required
sign is corpus-dependent and would be selected after viewing test labels. A declared
reference, not post-hoc sign choice, identifies orientation.

\section{Limitations and broader impact}
\label{sec:limitations}

Our conclusion concerns a raw positive centroid scored by cosine. It does not rule out
support estimators, Gaussian or other density models, positive--unlabeled methods,
safety-specialized encoders, or supervised guards. The study covers English responses,
two SBERT-family encoders, and one Qwen family at two scales under fixed pooling and
truncation; category-conditional performance and truncation by label are not audited.
BeaverTails and
PKU-SafeRLHF use released human annotations without a new adjudication study. Aegis is
jury-only in our primary arm and is auxiliary evidence. Only the population reference
in Section~\ref{sec:unlabeled} is built without labels: the cohort pools for BeaverTails
and PKU-SafeRLHF are selected by label (mixed-label prompts or discordant pairs), and
the 5\%, 10\%, and 25\% pools are subsampled by label; only their means enter the score.
The stylistic baseline covers simple surface features, not all style.

Bootstrap intervals hold the train-fitted direction fixed. They quantify test sampling
and, for operating points, validation-threshold variability, but omit train-refit
uncertainty. Thresholds are calibrated separately per corpus and method; they are not
portable deployment thresholds. The 5\% and 10\% targets are diagnostic and remain
too permissive for many high-stakes uses. We measure ranking and acceptance errors, not
changes in model behavior, attack robustness, or real-world harm.

\label{sec:ethics}
Exposing an unidentified shortcut can reduce false confidence in lightweight monitors;
the risk is that readers treat the referenced readout as a sufficient guard, so we
recommend it only as a diagnostic baseline. Licenses, compute, and provenance are in
Appendix~\ref{app:ethics}; the artifact is distributed as arXiv ancillary files.

\section{Conclusion}

Positive examples identify where a safe class lies, not which way safety increases.
Across controlled cohorts and four encoders, a raw safe centroid is weak or inverted and
below simple stylistic features, while an explicit reference improves ranking and, on two
of three corpora, calibrated safe acceptance; an unlabeled reference helps only to the
extent that it carries unsafe mass. Declare the reference, control prompt composition, compare
with nuisance baselines, and evaluate a held-out operating point before calling an
embedding score a safety direction.

\begin{ack}
We thank the anonymous FLMSec reviewers, whose comments led to the detector example,
all-encoder analyses, refusal, length, and stylistic controls, and held-out reference
tests in this version.
\end{ack}

\bibliographystyle{plainnat}
\bibliography{refs}

\clearpage
\raggedbottom
\appendix

\section{Ethics, licenses, and reproducibility details}
\label{app:ethics}

The public datasets contain potentially distressing harmful content. We conduct no new
human-subject data collection and do not redistribute prompts or responses. The
artifact contains code, aggregate metrics, row-order hashes, and source manifests only.
BeaverTails and PKU-SafeRLHF are used under CC BY-NC 4.0; Aegis 2.0 under CC BY 4.0.
The two sentence-transformer models and Qwen2.5-0.5B-Instruct are Apache 2.0;
Qwen2.5-3B-Instruct uses the non-commercial Qwen Research License. We redistribute no
model weights.

Each schema-v3 result records complete dataset/model revisions, ordered-row SHA-256
digests, source-file SHA-256 digests, dependency versions, command, split salt, and
bootstrap seed. Thirty-three protocol tests cover grouping, deduplication, overlap,
readouts, paired inference, and threshold calibration; the camera-ready analyses and
every appendix table are generated by tested scripts from those JSON files and fail if
a reproduced primary AUC differs from its frozen value. Runs used one NVIDIA RTX 4080
Laptop GPU (12.9 GB) and Python 3.13.1. Cold embedding time and total exploratory
compute were not completely logged, which we report as a reproducibility limitation.
The complete artifact is distributed with this paper as arXiv ancillary files.

\section{Readout and analysis details}
\label{app:details}

All readouts are fit on the primary train cohort and scored on the primary test cohort
of the same corpus unless stated otherwise. Embeddings are unit normalized, so every
dot product with a unit vector is a cosine. Let $\bar z_s$ and $\bar z_u$ be the train
safe and unsafe means and $\bar z$ the pooled train mean.

\begin{itemize}
    \item \textbf{Safe prototype:} $z^\top \bar z_s/\lVert\bar z_s\rVert$.
    \item \textbf{Difference of means:} $z^\top(\bar z_s-\bar z_u)/\lVert\bar z_s-\bar z_u\rVert$.
    \item \textbf{Two-centroid:} $z^\top \bar z_s/\lVert\bar z_s\rVert - z^\top \bar z_u/\lVert\bar z_u\rVert$.
    \item \textbf{Mixture-centered safe mean:} $z^\top(\bar z_s-\bar z)/\lVert\bar z_s-\bar z\rVert$.
    Because $\bar z_s-\bar z=\hat\pi_u(\bar z_s-\bar z_u)$ in sample, it ranks identically to
    difference of means; Section~\ref{sec:unlabeled} replaces $\bar z$ with a held-out mean.
    \item \textbf{RankOne:} for each train prompt with both labels, the unit vector of its safe
    mean minus its unsafe mean; the readout is the normalized average of these vectors.
    \item \textbf{Supervised linear:} scikit-learn 1.8.0 logistic regression with an L2 penalty,
    $C=1$, intercept, L-BFGS, tolerance $10^{-4}$, at most 2,000 iterations, and seed
    20260715, fit on the frozen embeddings without scaling or tuning.
    \item \textbf{Log response length:} $\log(1+\text{characters})$, multiplied by $+1$ or $-1$
    according to which class has the larger train mean.
    \item \textbf{Refusal-template indicator:} 1 if the first 120 characters of the response
    match an explicit refusal template (``I cannot/can't/won't/will not/am unable/am not
    able/must decline/refuse'', ``Sorry, (but) I'', ``I'm sorry, but'', or ``As an AI''),
    else 0, sign-oriented like length. The exact pattern is stored in the artifact.
    \item \textbf{Stylistic features:} logistic regression as above, after standardization,
    on $\log(1+\text{characters})$, $\log(1+\text{words})$, $\log(1+\text{lines})$, mean
    word length, type--token ratio, punctuation, digit, and uppercase fractions, the
    fraction of list-formatted lines, question-mark and exclamation counts, and the
    refusal indicator.
\end{itemize}

Operating-point thresholds accept a response as safe iff its score is at least the next
representable value above the validation unsafe score at rank
$\lfloor \text{target}\times n_{\text{unsafe}}\rfloor$ in descending order, which
guarantees a validation false-safe rate no greater than the target. The held-out
references (Section~\ref{sec:unlabeled}) use the validation split, which is
prompt-disjoint from train and test. The population reference uses every validation
response without labels. The cohort pool is the validation primary cohort. The 5\%,
10\%, and 25\% pools subsample the population, keeping all safe rows when possible and
subsampling unsafe rows (or the reverse when unsafe rows are scarce). The labeled
reference draws $k$ unsafe validation responses, where $k$ equals the 5\% pool's unsafe
count. All draws use seed 20260715.

\section{Per-cell results and statistics}
\label{app:percell}

\begin{table}[H]
\centering
\scriptsize
\begin{adjustbox}{max width=\textwidth}
\begin{tabular}{llccccccc}
\toprule
Corpus & Encoder & Prototype & Two-centroid & Diff.\ means & Mixture-centered & RankOne & Linear & Log length \\
\midrule
BeaverTails & MiniLM & .533 [.502,.559] & .591 & .588 [.559,.616] & .588 & .614 & .632 & .529 \\
BeaverTails & mpnet & .509 [.482,.536] & .594 & .590 [.566,.615] & .590 & .624 & .639 & .529 \\
BeaverTails & Qwen-0.5B & .545 [.514,.572] & .617 & .617 [.590,.643] & .617 & .631 & .690 & .529 \\
BeaverTails & Qwen-3B & .529 [.497,.556] & .622 & .622 [.596,.648] & .622 & .636 & .695 & .529 \\
\midrule
PKU-SafeRLHF & MiniLM & .491 [.481,.502] & .670 & .668 [.656,.683] & .668 & .667 & .733 & .595 \\
PKU-SafeRLHF & mpnet & .457 [.447,.467] & .699 & .690 [.678,.703] & .690 & .691 & .775 & .595 \\
PKU-SafeRLHF & Qwen-0.5B & .501 [.488,.514] & .716 & .716 [.703,.729] & .716 & .718 & .819 & .595 \\
PKU-SafeRLHF & Qwen-3B & .481 [.468,.494] & .738 & .738 [.726,.751] & .738 & .743 & .843 & .595 \\
\midrule
Aegis jury & MiniLM & .358 [.317,.395] & .739 & .772 [.739,.803] & .772 & N/A$^\dagger$ & .818 & .554 \\
Aegis jury & mpnet & .389 [.353,.427] & .764 & .793 [.764,.821] & .793 & N/A$^\dagger$ & .856 & .554 \\
Aegis jury & Qwen-0.5B & .381 [.339,.422] & .746 & .754 [.720,.789] & .754 & N/A$^\dagger$ & .843 & .554 \\
Aegis jury & Qwen-3B & .405 [.362,.448] & .756 & .761 [.726,.794] & .761 & N/A$^\dagger$ & .857 & .554 \\
\bottomrule
\end{tabular}
\end{adjustbox}
\caption{Primary controlled response-only ROC-AUC for every audited readout. Brackets are 95\% prompt-cluster bootstrap intervals (500 draws, conditional on the train fit). The in-sample mixture-centered safe mean equals difference of means in every cell, as Eq.~\ref{eq:mixture} predicts. $^\dagger$Aegis RankOne has only six mixed-label fit groups and is not interpreted.}
\label{tab:app-percell}
\end{table}

\begin{table}[H]
\centering
\scriptsize
\begin{adjustbox}{max width=\textwidth}
\begin{tabular}{llccccccc}
\toprule
Corpus & Encoder & Prototype AUC & $z$ & $p$ & Holm & TOST & $\Delta_+$ & $\cos(\mu_s,\mu_u)$ \\
\midrule
BeaverTails & MiniLM & .533 [.502,.559] & $+$2.29 & 2.2e-02 & no & no & $+$0.0044 & 0.947 \\
BeaverTails & mpnet & .509 [.482,.536] & $+$0.65 & 5.1e-01 & no & yes & $-$0.0042 & 0.951 \\
BeaverTails & Qwen-0.5B & .545 [.514,.572] & $+$3.04 & 2.4e-03 & yes & no & $+$0.0028 & 0.999 \\
BeaverTails & Qwen-3B & .529 [.497,.556] & $+$1.93 & 5.4e-02 & no & no & $+$0.0016 & 0.999 \\
\midrule
PKU-SafeRLHF & MiniLM & .491 [.481,.502] & $-$1.79 & 7.4e-02 & no & yes & $-$0.0053 & 0.967 \\
PKU-SafeRLHF & mpnet & .457 [.447,.467] & $-$8.32 & $<10^{-15}$ & yes & no & $-$0.0154 & 0.964 \\
PKU-SafeRLHF & Qwen-0.5B & .501 [.488,.514] & $+$0.21 & 8.4e-01 & no & yes & $+$0.0034 & 0.998 \\
PKU-SafeRLHF & Qwen-3B & .481 [.468,.494] & $-$2.90 & 3.7e-03 & yes & yes & $+$0.0008 & 0.999 \\
\midrule
Aegis jury & MiniLM & .358 [.317,.395] & $-$7.13 & 1.0e-12 & yes & no & $-$0.0346 & 0.837 \\
Aegis jury & mpnet & .389 [.353,.427] & $-$5.88 & 4.2e-09 & yes & no & $-$0.0227 & 0.825 \\
Aegis jury & Qwen-0.5B & .381 [.339,.422] & $-$5.60 & 2.1e-08 & yes & no & $-$0.0126 & 0.997 \\
Aegis jury & Qwen-3B & .405 [.362,.448] & $-$4.31 & 1.6e-05 & yes & no & $-$0.0062 & 0.998 \\
\bottomrule
\end{tabular}
\end{adjustbox}
\caption{Prototype-versus-chance tests and class-mean geometry. $z$ and two-sided $p$ use a normal approximation from the bootstrap interval half-width. Holm: Holm--Bonferroni rejection at family-wise $\alpha=0.05$ over the 12 cells. TOST: conservative equivalence to chance within $\pm0.05$ (95\% interval contained in $[0.45,0.55]$). $\Delta_+$ is the measured expected score gap of Eq.~\ref{eq:positive-gap} on the test cohort.}
\label{tab:app-stats}
\end{table}

\section{Operating points for all encoders}
\label{app:operating}

\begin{table}[H]
\centering
\scriptsize
\begin{adjustbox}{max width=\textwidth}
\begin{tabular}{lllcccc}
\toprule
Corpus & Encoder & Target & Prototype test FSR & Prototype safe accept. & Diff.\ means test FSR & Diff.\ means safe accept. \\
\midrule
BeaverTails & MiniLM & 5\% & .050 [.021,.074] & .082 [.038,.109] & .045 [.025,.085] & .084 [.056,.130] \\
 &  & 10\% & .091 [.057,.130] & .124 [.084,.171] & .104 [.068,.159] & .163 [.112,.229] \\
BeaverTails & mpnet & 5\% & .038 [.012,.071] & .037 [.014,.073] & .053 [.026,.082] & .092 [.061,.134] \\
 &  & 10\% & .090 [.043,.140] & .092 [.042,.153] & .102 [.061,.148] & .165 [.112,.221] \\
BeaverTails & Qwen-0.5B & 5\% & .044 [.025,.075] & .060 [.034,.100] & .034 [.021,.058] & .067 [.044,.109] \\
 &  & 10\% & .109 [.078,.150] & .143 [.096,.183] & .072 [.045,.106] & .147 [.093,.196] \\
BeaverTails & Qwen-3B & 5\% & .070 [.045,.112] & .068 [.047,.098] & .037 [.020,.065] & .095 [.062,.137] \\
 &  & 10\% & .127 [.100,.162] & .115 [.090,.151] & .082 [.053,.127] & .171 [.126,.236] \\
\midrule
PKU-SafeRLHF & MiniLM & 5\% & .064 [.047,.078] & .061 [.043,.077] & .052 [.039,.066] & .170 [.133,.198] \\
 &  & 10\% & .119 [.090,.147] & .117 [.089,.139] & .090 [.072,.115] & .250 [.222,.299] \\
PKU-SafeRLHF & mpnet & 5\% & .065 [.047,.081] & .039 [.028,.057] & .050 [.033,.071] & .153 [.113,.191] \\
 &  & 10\% & .115 [.087,.135] & .081 [.061,.097] & .098 [.079,.128] & .246 [.215,.291] \\
PKU-SafeRLHF & Qwen-0.5B & 5\% & .051 [.035,.070] & .048 [.035,.069] & .041 [.029,.057] & .230 [.186,.268] \\
 &  & 10\% & .103 [.082,.124] & .104 [.084,.124] & .087 [.067,.105] & .361 [.315,.397] \\
PKU-SafeRLHF & Qwen-3B & 5\% & .050 [.034,.064] & .050 [.032,.064] & .046 [.029,.061] & .263 [.222,.303] \\
 &  & 10\% & .097 [.078,.117] & .084 [.068,.101] & .099 [.071,.120] & .394 [.350,.431] \\
\midrule
Aegis jury & MiniLM & 5\% & .042 [.019,.066] & .019 [.004,.037] & .053 [.029,.085] & .245 [.160,.318] \\
 &  & 10\% & .081 [.055,.108] & .034 [.012,.063] & .117 [.081,.150] & .358 [.288,.476] \\
Aegis jury & mpnet & 5\% & .037 [.021,.065] & .004 [.000,.029] & .066 [.041,.095] & .291 [.205,.401] \\
 &  & 10\% & .072 [.051,.102] & .049 [.020,.076] & .126 [.092,.161] & .468 [.386,.548] \\
Aegis jury & Qwen-0.5B & 5\% & .060 [.037,.089] & .038 [.011,.065] & .044 [.027,.076] & .189 [.133,.292] \\
 &  & 10\% & .107 [.077,.138] & .075 [.043,.113] & .089 [.063,.127] & .351 [.278,.431] \\
Aegis jury & Qwen-3B & 5\% & .054 [.029,.074] & .045 [.018,.073] & .053 [.029,.078] & .208 [.147,.325] \\
 &  & 10\% & .086 [.059,.119] & .075 [.044,.110] & .089 [.064,.126] & .370 [.301,.454] \\
\bottomrule
\end{tabular}
\end{adjustbox}
\caption{Validation-calibrated operating points for all four encoders. Brackets are 95\% intervals from 10{,}000 validation/test prompt-cluster draws that recalibrate the threshold in each draw, conditional on a fixed train fit.}
\label{tab:app-operating}
\end{table}

\section{Input-mode ablation for all encoders}
\label{app:ablation}

\begin{figure}[H]
\centering
\includegraphics[width=0.8\textwidth]{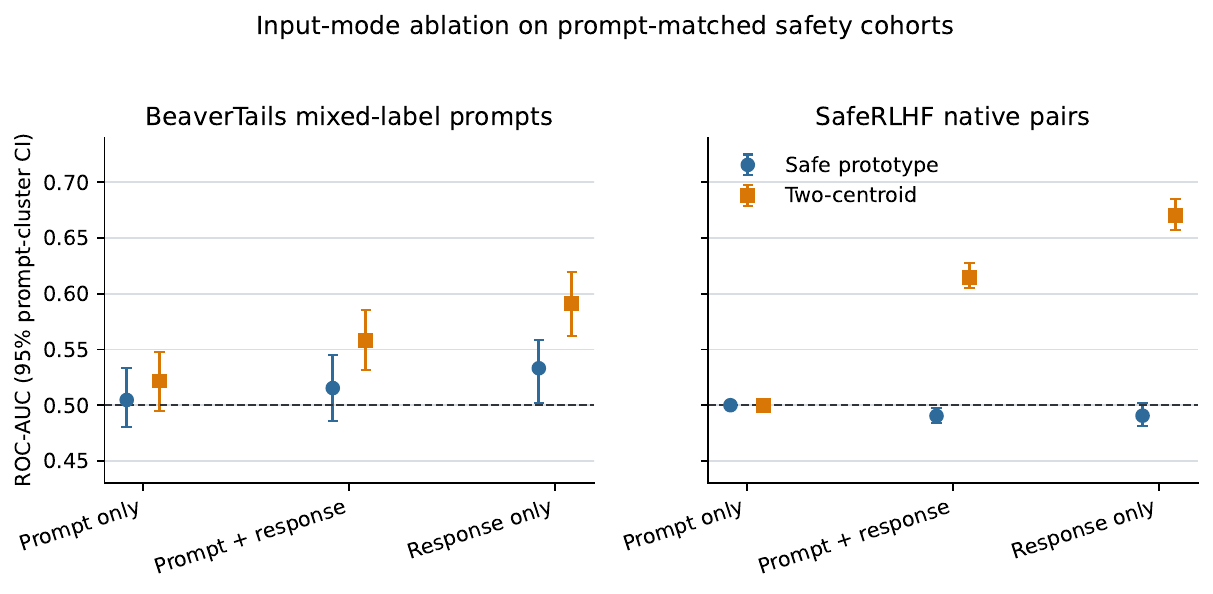}
\caption{MiniLM input-mode ablation on prompt-matched cohorts (``SafeRLHF'' is
PKU-SafeRLHF). Prompt-only scores are
at chance or near chance after matching, despite strong prompt-only signal on the full
corpus populations. Error bars are conditional 95\% prompt-cluster intervals. All four
encoders are reported in Table~\ref{tab:app-ablation}.}
\label{fig:input}
\end{figure}

\begin{table}[H]
\centering
\scriptsize
\begin{adjustbox}{max width=\textwidth}
\begin{tabular}{llcccc}
\toprule
Corpus & Encoder & Prompt only, full & Prompt only, controlled & Prompt+response (proto / 2c) & Response (proto / 2c) \\
\midrule
BeaverTails & MiniLM & .713 & .522 [.495,.547] & .515 / .558 & .533 / .591 \\
BeaverTails & mpnet & .731 & .531 [.503,.558] & .546 / .579 & .509 / .594 \\
BeaverTails & Qwen-0.5B & .679 & .543 [.515,.566] & .545 / .598 & .545 / .617 \\
BeaverTails & Qwen-3B & .692 & .548 [.521,.572] & .523 / .610 & .529 / .622 \\
\midrule
PKU-SafeRLHF & MiniLM & .780 & .500 [.500,.500] & .490 / .615 & .491 / .670 \\
PKU-SafeRLHF & mpnet & .831 & .500 [.500,.500] & .499 / .630 & .457 / .699 \\
PKU-SafeRLHF & Qwen-0.5B & .741 & .500 [.500,.500] & .506 / .680 & .501 / .716 \\
PKU-SafeRLHF & Qwen-3B & .768 & .500 [.500,.500] & .458 / .705 & .481 / .738 \\
\midrule
Aegis jury & MiniLM & .669 & .639 [.599,.677] & .337 / .689 & .358 / .739 \\
Aegis jury & mpnet & .670 & .649 [.611,.686] & .431 / .690 & .389 / .764 \\
Aegis jury & Qwen-0.5B & .662 & .643 [.601,.681] & .422 / .736 & .381 / .746 \\
Aegis jury & Qwen-3B & .642 & .621 [.580,.660] & .426 / .769 & .405 / .756 \\
\bottomrule
\end{tabular}
\end{adjustbox}
\caption{Input-mode ablation for all encoders, two-centroid (2c) readout unless marked. ``Full'' is the prompt-disjoint full population; ``controlled'' is the primary cohort (mixed-label prompts, native pairs, or the Aegis jury source control, which is not prompt-matched). Brackets are 95\% prompt-cluster intervals.}
\label{tab:app-ablation}
\end{table}

\section{Cross-corpus transfer for all encoders}
\label{app:transfer}

\begin{table}[H]
\centering
\scriptsize
\begin{adjustbox}{max width=\textwidth}
\begin{tabular}{llcc}
\toprule
Fit $\rightarrow$ evaluate & Encoder & Prototype & Diff.\ means \\
\midrule
BeaverTails $\rightarrow$ PKU-SafeRLHF & MiniLM & .524 [.512,.536] & .648 [.635,.661] \\
 & mpnet & .478 [.469,.490] & .657 [.646,.670] \\
 & Qwen-0.5B & .502 [.489,.515] & .690 [.679,.702] \\
 & Qwen-3B & .481 [.467,.495] & .698 [.687,.709] \\
\midrule
BeaverTails $\rightarrow$ Aegis jury & MiniLM & .275 [.240,.313] & .628 [.590,.663] \\
 & mpnet & .253 [.221,.285] & .682 [.645,.714] \\
 & Qwen-0.5B & .365 [.324,.405] & .577 [.534,.623] \\
 & Qwen-3B & .384 [.344,.423] & .559 [.518,.605] \\
\midrule
PKU-SafeRLHF $\rightarrow$ BeaverTails & MiniLM & .516 [.486,.543] & .608 [.577,.637] \\
 & mpnet & .501 [.469,.528] & .613 [.587,.639] \\
 & Qwen-0.5B & .544 [.513,.573] & .631 [.605,.654] \\
 & Qwen-3B & .523 [.491,.550] & .641 [.617,.665] \\
\midrule
PKU-SafeRLHF $\rightarrow$ Aegis jury & MiniLM & .306 [.273,.344] & .585 [.546,.622] \\
 & mpnet & .291 [.259,.323] & .662 [.626,.699] \\
 & Qwen-0.5B & .375 [.334,.414] & .551 [.507,.598] \\
 & Qwen-3B & .375 [.336,.413] & .548 [.504,.594] \\
\midrule
Aegis jury $\rightarrow$ BeaverTails & MiniLM & .514 [.488,.538] & .532 [.508,.558] \\
 & mpnet & .500 [.471,.529] & .560 [.535,.587] \\
 & Qwen-0.5B & .540 [.511,.568] & .549 [.520,.577] \\
 & Qwen-3B & .529 [.498,.558] & .553 [.526,.580] \\
\midrule
Aegis jury $\rightarrow$ PKU-SafeRLHF & MiniLM & .471 [.460,.484] & .544 [.535,.554] \\
 & mpnet & .448 [.437,.459] & .579 [.570,.589] \\
 & Qwen-0.5B & .492 [.479,.506] & .572 [.561,.584] \\
 & Qwen-3B & .471 [.457,.485] & .579 [.568,.593] \\
\bottomrule
\end{tabular}
\end{adjustbox}
\caption{Cross-corpus transfer between primary controlled cohorts after exact prompt decontamination, for all four encoders. Brackets are 95\% prompt-cluster intervals on the evaluation cohort.}
\label{tab:app-transfer}
\end{table}

\section{Refusal-template sensitivity}
\label{app:refusal}

\begin{table}[H]
\centering
\scriptsize
\begin{adjustbox}{max width=\textwidth}
\begin{tabular}{llccccc}
\toprule
Corpus & Encoder & Refusal \% (S / U) & Indicator AUC & Rows kept & Prototype, no refusals & Diff.\ means, no refusals \\
\midrule
BeaverTails & MiniLM & 5.1 / 0.6 & .523 & 1,814/1,868 & .523 [.497,.551] & .578 [.548,.605] \\
BeaverTails & mpnet & 5.1 / 0.6 & .523 & 1,814/1,868 & .501 [.472,.530] & .582 [.556,.608] \\
BeaverTails & Qwen-0.5B & 5.1 / 0.6 & .523 & 1,814/1,868 & .543 [.515,.573] & .603 [.574,.632] \\
BeaverTails & Qwen-3B & 5.1 / 0.6 & .523 & 1,814/1,868 & .534 [.502,.564] & .608 [.580,.636] \\
\midrule
PKU-SafeRLHF & MiniLM & 3.2 / 0.5 & .513 & 3,421/3,486 & .491 [.481,.501] & .664 [.650,.678] \\
PKU-SafeRLHF & mpnet & 3.2 / 0.5 & .513 & 3,421/3,486 & .457 [.448,.468] & .686 [.672,.698] \\
PKU-SafeRLHF & Qwen-0.5B & 3.2 / 0.5 & .513 & 3,421/3,486 & .505 [.493,.518] & .703 [.691,.714] \\
PKU-SafeRLHF & Qwen-3B & 3.2 / 0.5 & .513 & 3,421/3,486 & .488 [.474,.502] & .724 [.712,.736] \\
\midrule
Aegis jury & MiniLM & 0.8 / 2.0 & .506 & 932/948 & .357 [.317,.397] & .774 [.742,.810] \\
Aegis jury & mpnet & 0.8 / 2.0 & .506 & 932/948 & .389 [.350,.428] & .796 [.765,.828] \\
Aegis jury & Qwen-0.5B & 0.8 / 2.0 & .506 & 932/948 & .382 [.339,.424] & .756 [.721,.789] \\
Aegis jury & Qwen-3B & 0.8 / 2.0 & .506 & 932/948 & .407 [.364,.447] & .763 [.727,.800] \\
\bottomrule
\end{tabular}
\end{adjustbox}
\caption{Refusal-template sensitivity. A response is a refusal if its first 120 characters match an explicit refusal template (Appendix~\ref{app:details}). Refusal-opening responses are removed from both fit and test before refitting.}
\label{tab:app-refusal}
\end{table}

\section{Held-out unlabeled reference}
\label{app:unlabeled}

\begin{table}[H]
\centering
\scriptsize
\begin{adjustbox}{max width=\textwidth}
\begin{tabular}{llcccccccc}
\toprule
Corpus & Encoder & Prototype & Diff.\ means & Population & Cohort pool & 25\% & 10\% & 5\% & Labeled $k$ \\
\midrule
BeaverTails & MiniLM & .533 & .588 & .581 & .590 & .561 & .520 & .505 & .571 \\
BeaverTails & mpnet & .509 & .590 & .582 & .583 & .553 & .498 & .481 & .577 \\
BeaverTails & Qwen-0.5B & .545 & .617 & .600 & .605 & .573 & .543 & .530 & .597 \\
BeaverTails & Qwen-3B & .529 & .622 & .609 & .613 & .570 & .534 & .523 & .612 \\
\midrule
PKU-SafeRLHF & MiniLM & .491 & .668 & .623 & .643 & .556 & .516 & .506 & .625 \\
PKU-SafeRLHF & mpnet & .457 & .690 & .640 & .671 & .543 & .495 & .483 & .636 \\
PKU-SafeRLHF & Qwen-0.5B & .501 & .716 & .666 & .703 & .598 & .549 & .533 & .685 \\
PKU-SafeRLHF & Qwen-3B & .481 & .738 & .668 & .728 & .597 & .546 & .530 & .693 \\
\midrule
Aegis jury & MiniLM & .358 & .772 & .701 & .756 & .677 & .583 & .557 & .753 \\
Aegis jury & mpnet & .389 & .793 & .710 & .782 & .687 & .591 & .563 & .791 \\
Aegis jury & Qwen-0.5B & .381 & .754 & .563 & .737 & .540 & .474 & .460 & .728 \\
Aegis jury & Qwen-3B & .405 & .761 & .643 & .736 & .624 & .558 & .543 & .720 \\
\bottomrule
\end{tabular}
\end{adjustbox}
\caption{Held-out references (ROC-AUC on the primary test cohort). The safe mean is fit on labeled-safe train responses and centered by the mean of a prompt-disjoint validation sample: the full validation population (no labels used); the validation primary cohort (label-selected for BeaverTails and PKU-SafeRLHF); the population subsampled to 25\%, 10\%, or 5\% unsafe (labels used only to set composition); or $k$ labeled unsafe validation responses, with $k$ equal to the 5\% pool's unsafe count (Table~\ref{tab:app-unlabeled-geometry}).}
\label{tab:app-unlabeled}
\end{table}

\begin{table}[H]
\centering
\scriptsize
\begin{adjustbox}{max width=\textwidth}
\begin{tabular}{llccccccc}
\toprule
Corpus & Encoder & Population rows & $\hat\pi_u$ & $k$ & cos, population & cos, 5\% & cos, labeled $k$ & Shift \\
\midrule
BeaverTails & MiniLM & 4,603 & .568 & 105 & .888 & .139 & .783 & .855 \\
BeaverTails & mpnet & 4,603 & .568 & 105 & .873 & .085 & .788 & .848 \\
BeaverTails & Qwen-0.5B & 4,603 & .568 & 105 & .938 & .266 & .907 & .636 \\
BeaverTails & Qwen-3B & 4,603 & .568 & 105 & .929 & .288 & .895 & .685 \\
\midrule
PKU-SafeRLHF & MiniLM & 24,341 & .505 & 634 & .824 & .120 & .816 & 1.385 \\
PKU-SafeRLHF & mpnet & 24,341 & .505 & 634 & .782 & $-$.063 & .823 & 1.403 \\
PKU-SafeRLHF & Qwen-0.5B & 24,341 & .505 & 634 & .793 & .226 & .775 & .977 \\
PKU-SafeRLHF & Qwen-3B & 24,341 & .505 & 634 & .776 & .193 & .782 & .946 \\
\midrule
Aegis jury & MiniLM & 2,143 & .293 & 80 & .471 & .118 & .772 & .835 \\
Aegis jury & mpnet & 2,143 & .293 & 80 & .495 & .132 & .816 & .821 \\
Aegis jury & Qwen-0.5B & 2,143 & .293 & 80 & .164 & $-$.180 & .886 & 1.052 \\
Aegis jury & Qwen-3B & 2,143 & .293 & 80 & .250 & $-$.021 & .888 & 1.310 \\
\bottomrule
\end{tabular}
\end{adjustbox}
\caption{Reference pools and geometry. ``cos'' is the cosine between a reference direction and the difference-of-means direction. Shift is $\lVert\bar z^{\,train}_s-\bar z^{\,val}_s\rVert/\lVert\bar z^{\,train}_s-\bar z^{\,train}_u\rVert$: the difference between the safe means of the primary train cohort and the full validation population, in class-gap units. It reflects sampling and composition differences and does not depend on the reference's unsafe share.}
\label{tab:app-unlabeled-geometry}
\end{table}

\end{document}